\documentclass[a4paper]{article}
\usepackage{times,xspace}

\newcommand{\mathbold}[1]{\mathbf{#1}}

\usepackage{amsmath,amssymb,amsthm}
\usepackage[round]{natbib}
\usepackage{fullpage}
\usepackage[usenames,dvipsnames]{xcolor}
\usepackage[bookmarksdepth=3]{hyperref}
\hypersetup{%
    colorlinks=true,
    linkcolor=darkgray,
    citecolor=black,
   urlcolor  = black}

\newtheorem{theorem}{Theorem}[section]
\newtheorem*{theorem*}{Theorem}

\newtheorem{proposition}{Proposition}[section]
\newtheorem{corollary}[theorem]{Corollary}

\theoremstyle{definition}

\newtheorem{remark}[theorem]{Remark}

\newcommand{\ldbl}{\{\!\!\{}
\newcommand{\rdbl}{\}\!\!\}}
\newcommand{\bldbl}{\left\{\!\!\left\{}
\newcommand{\brdbl}{\right\}\!\!\right\}}
\newcommand{\Rb}{\mathbb{R}}
\newcommand{\Fb}{\mathbb{F}}

\newcommand{\Nb}{\mathbb{N}}

\newcommand{\Xb}{\mathbb{X}}
\newcommand{\set}[1] {\left\lbrace #1 \right\rbrace }

\newcommand{\walkrefinementb}[2]{#1_{\mathsf{W}[#2]}}
\newcommand{\wlrefinementname}{\mathsf{2\text{-}WL}}
\newcommand{\wlrefinement}[1]{#1_{\wlrefinementname}}

\newcommand{\wl}[1]{\mathsf{#1\text{-}WL}}
\newcommand{\walk}[1]{\mathsf{W[#1]}}

\newcommand{\MPNN}{MPNN\xspace}
\newcommand{\MPNNs}{MPNNs\xspace}
\newcommand{\GNN}{GNN\xspace}
\newcommand{\GNNs}{GNNs\xspace}

\newcommand{\GINs}{GINs\xspace}
\newcommand{\MLP}{MLP\xspace}
\newcommand{\MLPs}{MLPs\xspace}

\title{Walk Message Passing Neural Networks and Second-Order Graph Neural Networks}

\author{
  Floris Geerts\\
  University of Antwerp\\
  \texttt{floris.geerts@uantwerp.be}
}

\date{}
\begin{document}

\maketitle

\begin{abstract}
	The expressive power of message passing neural networks (\MPNNs) is known to match the expressive power of the 1-dimensional Weisfeiler-Leman graph ($\wl{1}$) isomorphism test. To boost the expressive power of \MPNNs, a number of graph neural network architectures have recently been proposed based on higher-dimensional Weisfeiler-Leman tests. In this paper we consider the two-dimensional ($\wl{2}$) test and introduce a new type of \MPNNs, referred to as  $\ell$-walk \MPNNs, which aggregate features along walks of length $\ell$ between vertices. We show that $2$-walk MPNNs match $\wl{2}$ in expressive power. More generally, $\ell$-walk \MPNNs, for any $\ell\geq 2$, are shown to match the expressive power of the recently introduced $\ell$-walk refinement procedure ($\walk{\ell}$). Based on a correspondence between $\wl{2}$ and $\walk{\ell}$, we observe that $\ell$-walk \MPNNs and $2$-walk \MPNNs have the same expressive power, i.e., they can distinguish the same pairs of graphs, but $\ell$-walk \MPNNs can possibly distinguish pairs of graphs faster than $2$-walk \MPNNs.
	
	When it comes to concrete learnable graph neural network (\GNN) formalisms that match $\wl{2}$ or $\walk{\ell}$ in expressive power, we consider second-order graph neural networks that allow for non-linear layers. In particular, to match $\walk{\ell}$ in expressive power, we allow $\ell-1$ matrix multiplications in each layer. We propose different versions of second-order \GNNs depending on the type of features (i.e., coming from a countable set, or coming from an uncountable set) as this affects the number of dimensions needed to represent the features.
Our results indicate that increasing non-linearity in layers by means of allowing multiple matrix multiplications does not increase expressive power. At the very best, it results in a faster distinction of input graphs.
 \end{abstract}

\section{Introduction}
One of the most popular methods for  deep learning on graphs are the message passing neural networks (\MPNNs)  introduced by~\citet{GilmerSRVD17}. An \MPNN iteratively propagates vertex features based on the adjacency structure of a graph in a number of rounds. In each round, every vertex receives messages from its neighbouring vertices, based on the features computed in the previous round. Then, each vertex aggregates the received messages and performs an additional update based on the feature of the vertex itself. As such, new features are obtained for every vertex and the \MPNN proceeds to the next round. When the features consist of tuples in $\Rb^{n}$, an \MPNN can be regarded as a means of computing an embedding of the vertices of a graph into $\Rb^n$. An \MPNN can also include an additional read-out phase in which the embedded vertices are combined to form a single representation of the entire graph.
Important questions in this context relate to the expressive power of \MPNNs, such as: ``When can two vertices be distinguished by means of the computed embedding?'' and ``When can two graphs be distinguished?''.

In two independent works~\citep{grohewl,xhlj19} such expressivity questions were addressed by connecting \MPNNs to the one-dimensional Weisfeiler-Leman  ($\wl{1}$) graph isomorphism test. Alike \MPNNs,  $\wl{1}$ also iteratively updates
vertex features based on the graph's adjacency structure. 
\citet{grohewl} and~\citet{xhlj19} show that \MPNNs cannot distinguish more vertices by means of the computed embeddings than $\wl{1}$ does. In other words, the expressive power of \MPNNs is bounded by $\wl{1}$. 

Furthermore, \citet{grohewl} identify a simple class of \MPNNs that is as expressive as $\wl{1}$. In other words, for every graph there exists an \MPNN in that class whose distinguishing power matches that of $\wl{1}$. Similarly, by applying \MPNNs on the direct sum of two graphs, these \MPNNs can only distinguish the component graphs when $\wl{1}$ can distinguish them. In \citet{geerts2020lets}, similar results were established for an even simpler class of \MPNNs and generalised to MPNNs that that can use degree information (such as the graph convolutional networks by \citet{kipf-loose}). There is a close correspondence between $\wl{1}$ and logic. More precisely, two graphs are indistinguishable by $\wl{1}$ if and only if no sentence in the  two-variable fragment of first-order logic with counting can distinguish those graphs. A more refined analysis of \MPNNs based on this connection to logic can be found in~\citet{Barcel2020TheLE}. The impact of random features on the expressive power of \MPNNs is
considered in~\citet{sato2020random}.

\citet{xhlj19} propose another way of letting \MPNNs match the expressive power of $\wl{1}$. More specifically, they
propose so-called graph isomorphism networks (\GINs) and show that \GINs can distinguish any two graphs (in some collection
of graphs) whenever $\wl{1}$ does so. \GINs crucially rely on the use of multi layer perceptrons (\MLPs) and their universality~\citep{Cybenko92,Hornik91}. To leverage this universality, the collection of graphs should have bounded degree and all features combined should originate from a finite set.

Since $\wl{1}$ fails to distinguish even very simple graphs 
the above results imply that  \MPNNs have  limited expressive power. To overcome this limitation, higher-dimensional Weisfeiler-Leman graph isomorphism tests have recently be considered as inspiration for constructing graph embeddings.
For a given dimension $k$, the $\wl{k}$\footnote{What we refer to as $\wl{k}$ is sometimes referred to as the ``folklore'' $k$-dimensional Weisfeiler-Leman test.} test iteratively propagates features for $k$-tuples of vertices and again relies on the adjacency structure of the graph \citep{grohe_otto_2015,grohe_2017}.  From a logic perspective, two graphs are indistinguishable by $\wl{k}$ if and only if they are indistinguishable by sentences in the $(k+1)$-variable fragment of first-order logic with counting and their expressive power is known to increase with increasing $k$ \citep{CaiFI92}.

The focus of this paper on $\wl{2}$. By using a graph product construction, \MPNNs can be used to match the distinguishing power of $\wl{2}$ \citep{grohewl}. The vertices on which the \MPNN act are now triples of vertices and a notion of adjacency between such triples is considered\footnote{To be more precise: a set-based version of $\wl{2}$ was considered in \citet{grohewl} where ``vertices'' $(u,v,w)$ correspond to a set of three vertices $\{u,v,w\}$, and two 
vertices $(u,v,w)$ and $(u',v',w')$ are adjacent if and only if $|\{u,v,w\}\cap \{u',v',w'\}|=2$.
}. A disadvantage of this approach is that one has to deal with $\mathcal{O}(n^{3})$ many embeddings. On the positive side, the dimension of the features is $\mathcal{O}(n^2)$. 
More closely in spirit to \GINs, \citet{maron2018invariant} introduced higher-order (linear) invariant graph neural networks (\GNNs) that use third-order tensors in $\Rb^{n^3\times s}$ and \MLPs to simulate $\wl{2}$ \citep{DBLP:conf/nips/MaronBSL19} . Also here, $\mathcal{O}(n^{3})$ many embeddings are used. It is not known whether  third-order \GNNs are also bounded in expressive power by $\wl{2}$\footnote{We remark that it has recently been shown in \citet{chen2020graph} that second-order linear \GNNs are bounded in expressive power by $\wl{1}$ on undirected graphs.}.
 We remark that the constructions provided in \citet{grohewl} and \citet{DBLP:conf/nips/MaronBSL19}  generalise to $\wl{k}$ by using multiple graph products and higher-order tensors, respectively. A more detailed overview of these approaches and results can be found in the recent survey by \citet{Sato2020ASO}. 

Perhaps the most promising approach related to $\wl{2}$ is the one  presented in \citet{DBLP:conf/nips/MaronBSL19}.
In that paper, simple second-order invariant \GNNs are introduced, using second-order tensors in $\Rb^{n^2 \times s}$ and \MLPs, which can simulate $\wl{2}$. A crucial ingredient in these networks is that the layers are non-linear. More specifically, the non-linearity stems from the use of a single matrix multiplication in each layer. This approach only requires to deal  with $\mathcal{O}(n^2)$ many embeddings making them more applicable than previous approaches. The downside is that the dimension of features needed increases in each round.  In this paper we
zoom in into those second-order non-linear \GNNs and aim to provide some deeper insights. The contributions made in this paper can be summarised as follows.
\begin{enumerate}
\item \label{cont:one} We first introduce $\ell$-walk \MPNNs in order to model second-order non-linear invariant \GNNs. Walk \MPNNs 
operate on pairs of vertices and can aggregate feature information along walks of a certain length $\ell$ in graphs. We show that
$\ell$-walk \MPNNs are bounded in expressive power by the $\ell$-walk refinement procedure ($\mathsf{W[\ell]}$) recently introduced by~\citet{lichter2019walk}. Furthermore, we show that $\ell$-walk \MPNNs match the expressive power of $\mathsf{W[\ell]}$.

\item We verify that second-order non-linear invariant \GNNs  are instances of $2$-walk 
\MPNNs. A direct consequence is that their expressive power is bounded by $\mathsf{W[2]}$ which is known to correspond to $\wl{2}$~\citep{lichter2019walk}. Intuitively,
walks of length two correspond to the use of a single matrix multiplication in \GNNs\footnote{We recall that for an adjacency matrix $\mathbf{A}_G$ of a graph $G$, the entries in $\mathbf{A}_G^\ell$ correspond to the number of walks of length $\ell$ between pairs of vertices.}. We recall from \citet{DBLP:conf/nips/MaronBSL19} 
that  second-order non-linear invariant \GNNs are also as expressive as $\wl{2}$.

\item We generalise second-order non-linear invariant \GNNs by allowing $\ell-1$ matrix multiplications in each layer, for $\ell\geq 2$, and
verify that these networks can be seen as instances of $\ell$-walk \MPNNs. They are thus bounded in expressive power by $\mathsf{W[\ell]}$. We generalise the construction given in \citet{DBLP:conf/nips/MaronBSL19} and show that they also match $\mathsf{W[\ell]}$ in expressive power.

\item Based on the properties of $\mathsf{W[\ell]}$ and $\wl{2}$ reported in~\citet{lichter2019walk}, we observe that allowing for multiple matrix multiplications does not increase the expressive power of second-order \GNNs, but vertices and graphs can potentially be distinguished faster (in a smaller number of rounds) than when using only a single matrix multiplication.

\item In order to reduce the feature dimensions needed we consider the setting in which the features are taken from a countable domain, just as in~\citet{xhlj19}. In this setting, we observe that a constant feature dimension suffices to model $\wl{2}$ and $\mathsf{W}[\ell]$. We recall than when the features are taken from the reals, the second-order \GNNs mentioned earlier require  increasing feature dimensions in each round, just as in \citet{DBLP:conf/nips/MaronBSL19}. We obtain learnable architectures, similar to \GINs, matching $\walk{\ell}$ in expressive power.

\item Finally, we show that the results in~\citet{grohewl} can be generalised by using non-linearity.
As a consequence, we obtain a simple form of $\ell$-walk MPNNs that can simulate $\walk{\ell}$ (and thus also $\wl{2}$)
on a given graph using only $\mathcal{O}(n^2)$ many embeddings. We recall that the higher-order graph neural networks in~\citet{grohewl} require $\mathcal{O}(n^3)$ many embeddings. Furthermore, we preserve the nice property that the dimension of the features is of size $\mathcal{O}(n^2)$.
\end{enumerate}

Our results can be seen as partial answer to the question raised by \citet{openprob}, whether polynomial layers (of degree greater than two) increase the expressive power of second-order invariant \GNNs. We answer this negatively in the restricted setting in which each layer consists of multiple matrix multiplications rather than general equivariant polynomial layers. Indeed, the use of multiple  matrix multiplications can be simulated by single matrix multiplication at the cost of introducing additional layers. 

For readers familiar with \GNNs we summarise the proposed architectures in Table~\ref{tbl:GNNs} and refer for details  to Section~\ref{sec:GNNs}. All architectures generalise to match  $\walk{\ell}$ in expressive power. We note that the last architecture in Table~\ref{tbl:GNNs} is the one proposed by \citet{DBLP:conf/nips/MaronBSL19}.
\begin{table}
$$\begin{array}{ll}\hline
 \text{Dimensions of } \textbf{A}^{(t)} & \text{GNN}\\	\hline\hline
n\times n&	\mathbf{A}^{(t)}_{ij}:= \sum_{k\in[n]} \mathsf{MLP}_{\mathbold{\theta}^{(t)}}\bigl(\mathbf{A}^{(t-1)}_{ik},\mathbf{A}^{(t-1)}_{kij}\bigr)\\
n\times n&	\mathbold{A}^{(t)}_{ij}:=\sum_{k\in [n]} \mathsf{MLP}_{\mathbold{\theta}^{(t)}}\left(\mathsf{MLP}_{\mathbold{\theta}^{(t)}_1}(\mathbold{A}^{(t-1)})_{ik}\cdot \mathsf{MLP}_{\mathbold{\theta}^{(t)}_2}(\mathbold{A}^{(t-1)})_{kj}\right)\\
n\times n\times 2&\mathbold{A}^{(t)}_{ijs}:=\mathsf{MLP}_{\theta^{(t)}_1}\left(
\sum_{k\in [n]} \mathsf{MLP}_{\theta^{(t)}_2}\left(\mathbold{A}^{(t-1)}_{ik1}\cdot \mathbold{A}^{(t-1)}_{kj2}\right)\right)\\
 n\times n\times s_t, s_t\in\mathcal{O}(n^2)&	\mathbf{A}^{(t)}_{ijs}:=\mathsf{ReLU}\left( \sum_{k\in[n]}\sum_{c,d\in [s_{t-1}]}\mathbf{A}^{(t-1)}_{ikc}\cdot\mathbf{A}^{(t-1)}_{kjd}\cdot
	\mathbf{W}^{(t)}_{cds} - q \mathbf{J}_{ijs}\right)\\
 n\times n\times s_t, s_t={n+s_{t-1}\choose s_{t-1}}&\mathbold{A}^{(t)}_{ijs}:=\sum_{k\in[n]} \mathsf{MLP}_{\mathbold{\theta}^{(t)}_1}(\mathbold{A}^{(t-1)})_{iks}\cdot \mathsf{MLP}_{\mathbold{\theta}^{(t)}_2}(\mathbold{A}^{(t-1)})_{kjs}\\\hline
\end{array}$$
\vspace{-3ex}
\caption{Various graph neural network architectures matching $\wl{2}$ in expressive power.}\label{tbl:GNNs}
\end{table}

\paragraph{Organisation of the paper.}
We start by introducing notation and  describing the $2$-dimensional Weisfeiler-Leman ($\wl{2}$) graph isomorphism test and walk refinement procedure ($\walk{\ell}$) in Section~\ref{sec:preliminaries}. To model $\walk{\ell}$ as a kind of \MPNN we introduce $\ell$-walk \MPNNs in Section~\ref{sec:walkMPNN}. In Section~\ref{sec:upperb} we verify that $\ell$-walk \MPNNs are bounded in expressive power by $\walk{\ell}$. Matching lower bounds on the expressive power of $\ell$-walk \MPNNs are provided in Section~\ref{subsec:lowerb} in the case when labels originate from a countable domain, and when they come from an uncountable domain. The obtained insights are used in Section~\ref{sec:GNNs} to build learnable graph neural networks that match $\walk{\ell}$ (and $\wl{2}$ in particular) in expressive power. We conclude the paper in Section~\ref{sec:conclude}.

\section{Preliminaries}\label{sec:preliminaries}
We use $\set{}$ and $\ldbl\rdbl$ to indicate sets and multisets, respectively. 
The sets of natural, rational, and real numbers are denoted by $\Nb$, $\mathbb{Q}$, and $\Rb$, respectively. We write $\Fb^+$ to denote the subset of numbers
from $\Fb$ which are strictly positive, e.g., $\Nb^+=\Nb\setminus\{0\}$.
For $n\in\Nb^+$, we denote with $[n]$ the set of numbers $\{1, \dots , n\}$.

\paragraph{Labelled graphs.} A labelled directed graph is given by ${G = (V, E, \mathbold{\eta})}$ with vertex set $V$, edge relation $E\subseteq V^2$, and where $\mathbold{\eta}\colon E \to \Sigma$ is an edge labelling function into some set~$\Sigma$ of labels.  Without loss of generality we identify $V$ with $[n]$. For $\ell\in\Nb^+$, a walk in $G$ from vertex $i$ to vertex $j$ of length $\ell$ is a sequence of vertices $(i,i_1,i_2,\ldots,i_{\ell-1},j)$ such that each consecutive pair of vertices is an edge in $G$. For $\ell\in\Nb^+$ we denote by $\mathsf{W}^{\ell}_G(i,j)$ the set of walks 
of length $\ell$ in $G$ starting in $i$ and ending at $j$. 

\begin{remark}\label{remark:edgevsvertex}
We opt to work with edge-labelled graphs rather than the more standard vertex-labelled graphs. This does not impose any restriction since we can always turn a vertex-labelled graph into an edge-labelled graph. More specifically, given a vertex-labelled graph $G=(V,E,\mathbold{\nu})$ with $\mathbold{\nu}:V\to\Sigma$ one can define the corresponding edge-labelling $\mathbold{\eta}:E\to \Sigma\times\Sigma$ by  $\mathbold{\eta}(i,j):=\bigl(\mathbold{\nu}(i),\mathbold{\nu}(j)\bigr)$, and then simply consider $G=(V,E,\mathbold{\eta})$ instead of $G=(V,E,\mathbold{\nu})$.  \qed
\end{remark}

\paragraph{Refinements of labellings.} We will need to be able to compare two edge labellings and we do this as follows.
Given two labellings $\mathbold{\eta}:E\to\Sigma$ and $\mathbold{\eta}':E\to\Sigma'$ we say that $\mathbold{\eta}$ \textit{refines} $\mathbold{\eta}'$, denoted by $\mathbold{\eta}\sqsubseteq \mathbold{\eta}'$, if for every 
$(i,j)$ and $(i',j')\in E$, $\mathbold{\eta}(i,j)=\mathbold{\eta}(i',j')$ implies that $\mathbold{\eta}'(i,j)=\mathbold{\eta}'(i',j')$. If $\mathbold{\eta}\sqsubseteq\mathbold{\eta}'$ and $\mathbold{\eta}'\sqsubseteq\mathbold{\eta}$ hold, then $\mathbold{\eta}$ and $\mathbold{\eta}'$ are said to be \textit{equivalent}, and we denote this by $\mathbold{\eta}\equiv\mathbold{\eta}'$.

We next describe two procedures which iteratively generate refinements of edge labellings.
First, we consider the \textit{2-dimensional Weisfeiler-Leman} ($\wl{2}$) procedure. This procedure iteratively
generates edge labellings,  starting from an initial labelling $\mathbold{\eta}$, until no further changes to the edge labelling is made. The labelling produced in round $t$ is denoted  by $\wlrefinement{\mathbold{\eta}}^{(t)}$. Since $\wl{2}$
generates labellings for all pairs of vertices, it is commonly assumed that the input graph is a complete graph, i.e., $E=V^2$. We remark that an incomplete graph $G=(V,E,\mathbold{\eta})$ can always be regarded as a complete graph in which the (extended) edge labelling $\mathbold{\eta}:V^2\to\Sigma$ assigns a special label to non-edges, i.e., those pairs in $V^2\setminus E$.

Let ${G = (V, E, \mathbold{\eta})}$ be a (complete) labelled graph. Then the initial labelling produced by $\wl{2}$ is defined as $\wlrefinement{\mathbold{\eta}}^{(0)}:=\mathbold{\eta}$. For $t>0$ and $i,j\in[n]$ we define:
\[
\wlrefinement{\mathbold{\eta}}^{(t)}(i,j):=\textsc{Hash}\left(\wlrefinement{\mathbold{\eta}}^{(t-1)}(i,j),\bldbl\bigl(\wlrefinement{\mathbold{\eta}}^{(t-1)}(i,k),\wlrefinement{\mathbold{\eta}}^{(t-1)}(k,j)\bigr)\mid k\in[n]\brdbl\right),
\]
where $\textsc{Hash}$ injectively maps $(a,S)$ with $a\in\Sigma$ and $S$ a multiset of pairs of labels in $\Sigma$ to a unique label in $\Sigma$. 
It is known that $\wlrefinement{\mathbold{\eta}}^{(t)}\sqsubseteq \wlrefinement{\mathbold{\eta}}^{(t-1)}$, for all $t>0$, and thus the $\wl{2}$ procedure indeed generates refinements of labellings.
We denote by 
$\wlrefinement{\mathbold{\eta}}^{\phantom{(t)}}$ the labelling $\wlrefinement{\mathbold{\eta}}^{(t)}$ such that  $\wlrefinement{\mathbold{\eta}}^{(t+1)}\equiv \wlrefinement{\mathbold{\eta}}^{(t)}$. It is known that $\wlrefinement{\mathbold{\eta}}^{\phantom{(t)}}$ is reached using at most $t={\cal O}(n\log(n))$ rounds, where $n=|V|$~\citep{lichter2019walk}.

One can simplify $\wl{2}$ by assuming that the initial labelling $\mathbold{\eta}$ assigns different labels
to loops (i.e., pairs of the form $(i,i)$ for $i\in [n]$) than it does to other edges. In other words, when for every $i,j,k\in[n]$ such that $j\neq k$, ${\mathbold{\eta}(i,i) \neq \mathbold{\eta}(j,k)}$ holds. Under this assumption, one can equivalently consider:
$$ \wlrefinement{\mathbold{\eta}}^{(t)}(i,j):=\textsc{Hash}\left(\bldbl\bigl(\wlrefinement{\mathbold{\eta}}^{(t-1)}(i,k),\wlrefinement{\mathbold{\eta}}^{(t-1)}(k,j)\bigr)\mid k\in[n]\brdbl\right).$$
In the following, we always assume that $\mathbold{\eta}$ treats loops differently from non-loops. One can always ensure this by modifying the labels of a given edge labelling.

To make $\wl{2}$ invariant under graph isomorphisms one additionally requires that the initial edge-labelling respects
transpose equivalence, i.e., for any
	 $i,j,i',j'\in[n]$, $\mathbold{\eta}(i,j)=\mathbold{\eta}(i',j')$ implies that $\mathbold{\eta}(j,i)=\mathbold{\eta}(j',i')$. In the following we always assume that this assumption holds. One
	 can again ensure this by applying an appropriate modification to a given edge labelling. We also note that this assumption is satisfied when the edge labelling originates from a vertex labelling, as explained in Remark~\ref{remark:edgevsvertex}. 

The second  procedure which we consider is the \textit{$\ell$-walk refinement procedure} ($\walk{\ell}$), recently introduced by \citet{lichter2019walk}. Similar to $\wl{2}$, it iteratively generates labellings. The labelling produced by $\walk{\ell}$ in round $t$ is denoted by $\walkrefinementb{\mathbold{\eta}^{(t)}}{\ell}$. The initial labelling is defined as $\walkrefinementb{\mathbold{\eta}^{(0)}}{\ell}:=\mathbold{\eta}$, just as for $\wl{2}$. For $t>0$ and 
$i,j\in [n]$ we define:
$$
\walkrefinementb{\mathbold{\eta}^{(t)}}{\ell}(i,j)
:=\textsc{Hash}\left(\bldbl\bigl(\walkrefinementb{\mathbold{\eta}^{(t-1)}}{\ell}(i,i_1),\walkrefinementb{\mathbold{\eta}^{(t-1)}}{\ell}(i_1,i_2),\ldots,\walkrefinementb{\mathbold{\eta}^{(t-1)}}{\ell}(i_{\ell-1},j)\bigr)\mid i_1,\ldots,i_{\ell-1}\in[n]\brdbl\right),
$$
where $\textsc{Hash}$ now injectively maps multisets of $\ell$ pairs of labels in $\Sigma$ to a unique label in $\Sigma$. 

We observe that $\wlrefinement{\mathbold{\eta}}^{(t)}=\walkrefinementb{\mathbold{\eta}}{2}^{(t)}$. Furthermore, for every $t>0$, 
$\walkrefinementb{\mathbold{\eta}}{\ell}^{(t)}\sqsubseteq \walkrefinementb{\mathbold{\eta}}{\ell}^{(t-1)}$ and thus also  $\walk{\ell}$ generates refinements of labellings. We define $\walkrefinementb{\mathbold{\eta}^{\phantom{(t)}}}{\ell}$ as 
the labelling $\walkrefinementb{\mathbold{\eta}}{\ell}^{(t)}$ such that $\walkrefinementb{\mathbold{\eta}}{\ell}^{(t+1)}\equiv\walkrefinementb{\mathbold{\eta}}{\ell}^{(t)}$. We further recall from
\citet{lichter2019walk} that $\walkrefinementb{\mathbold{\eta}}{\ell}^{(t)}\sqsubseteq \walkrefinementb{\mathbold{\eta}}{k}^{(t)}$
for $k\leq\ell$ and that $\wlrefinement{\mathbold{\eta}}^{(t\lceil\log \ell\rceil)}\sqsubseteq \walkrefinementb{\mathbold{\eta}}{\ell}^{(t)}$ for all $t\geq 0$. In particular, $\walkrefinementb{\mathbold{\eta}^{\phantom{(t)}}}{\ell}\equiv \wlrefinement{\mathbold{\eta}}^{\phantom{(t)}}$.

\begin{remark}
We thus see that both procedures generate the same labelling after a (possibly different) number of rounds. The labellings obtained by the two procedures may be different, however, in each round, except for $t=0$, as is illustrated in \citet{lichter2019walk}. Furthermore, if $\wlrefinement{\mathbold{\eta}}^{\phantom{(t)}}$ is reached in $T$ rounds by the $\wl{2}$ procedure, then it is reached in $T/\lceil\log\ell\rceil$ rounds by the $\walk{\ell}$ procedure.
\qed
\end{remark}

\paragraph{Labellings and matrices.}
Given a tensor $\mathbf{A}\in\mathbf{R}^{n^2\times s}$ we denote by $\mathbf{A}_{ijk}\in\Rb$ its
entry at position $i,j\in[n]$ and $k\in [s]$,  by $\mathbf{A}_{ij\bullet}\in\Rb^{s}$ the vector at position
$i,j\in[n]$, and by $\mathbf{A}_{i\bullet\bullet}\in\Rb^{n\times s}$ the matrix at position
$i\in[n]$. Similar notions are in place for matrices and higher-order tensors.
 A tensor $\mathbf{A}\in\mathbf{R}^{n^2\times s}$ naturally corresponds to an edge labelling $\mathbold{\eta}:E\to\Rb^s$ by letting $\mathbold{\eta}(i,j):=\mathbold{A}_{ij\bullet}$ for $i,j\in[n]$. 
Conversely, when given an edge labelling $\mathbold{\eta}:E\to\Sigma$, for $E=V^2$, we assume that we can
encode the labels in $\Sigma$ as vectors in some $\Rb^s$.
A common way to do this is by hot-one encoding labels in $\Sigma$ by basis vectors in $\Rb^s$ for some $s\in\Nb$.
In this way, $\mathbold{\eta}$ can be regarded as a tensor
in $\Rb^{n^2\times s}$.  We interchangeably
consider edge labels and edge labellings as vectors and tensors, respectively. 

\section{Walk Message Passing Neural Networks}\label{sec:walkMPNN}
We start by extending  \MPNNs such that they can easily model the walk-refinement procedure described above. This generalisation of \MPNNs is such that  message passing occurs between pairs of vertices and is restricted by walks in graphs, rather than between single vertices and their adjacent vertices as in  standard \MPNNs \citep{GilmerSRVD17}. We will refer to this generalisation as \textit{walk} \MPNNs. 

Walk \MPNNs iteratively compute edge labellings starting from an input labelled graph $G =(V, E, \mathbold{\eta})$. We refer to each iteration as a \textit{round}.
Walk \MPNNs are parametrised by a number $\ell\in\Nb$, with $\ell\geq 2$, which bounds the length of walks considered, and
we refer to them as \textit{$\ell$-walk} \MPNNs. 
We assume that the edge labelling of the input graph is of the form $\mathbold{\eta}:E\to \Rb^{s_0}$ for some $s_0\in\Nb^+$. In what follows we fix the number of vertices to be $n$.

After round $t\geq 0$, the labelling returned by an $\ell$-walk \MPNN  $M$ is denoted by
$\mathbold{\eta}^{(t)}_M$ and is of the form $\mathbold{\eta}^{(t)}_M:E\to\Rb^{s_t}$, for some $s_t\in\Nb^+$. We omit the dependency on the input graph $G$ in the labellings unless specified otherwise.
We next detail how $\mathbold{\eta}^{(t)}_M$ is computed.
\begin{description}\setlength{\itemsep}{-0.4ex}
\item [Initialisation.]  We let $\mathbold{\eta}^{(0)}_M:=\mathbold{\eta}$.
\end{description}
Then, for every round $t=1,2,\ldots$ we define
$\mathbold{\eta}^{(t)}_M:E\to\Rb^{s_t}$, as follows:
\begin{description}\setlength{\itemsep}{-0.4ex}
\item [Message Passing.] 
Each pair $(v,w)\in V^2$ receives messages from ordered sequences of edges on walks in $G$ of length
$\ell$ starting in $v$ and ending at $w$. These messages are
subsequently aggregated. Formally, if $(v,v_1,\ldots,v_{\ell-1},w)$ is a walk of length $\ell$ in $G$ then
the function $\textsc{Msg}^{(t)}$ receives the labels (computed in the previous round)  $\mathbold{\eta}^{(t-1)}_M(v,v_1)$,  
$\mathbold{\eta}^{(t-1)}_M(v_{1},v_{2}),\ldots, \mathbold{\eta}^{(t-1)}_M(v_{\ell-1},w)$ of the edges in this walk, 
 and outputs a label in $\Rb^{s'_t}$, for some $s_t'\in\Nb^+$. Then, for
every pair $(v,w)\in V^2$ we aggregate by summing all the received labels:
$$
\mathbold{m}^{(t)}_M(v,w):=\sum_{(v,v_1,\dots,v_{\ell-1},w)\in
\mathsf{W}_G^\ell(v,w)}\!\!\!\!\!\!\!\!\textsc{Msg}^{(t)}_M\left(\mathbold{\eta}^{(t-1)}_M(v,v_1),\mathbold{\eta}^{(t-1)}_M(v_1,v_2),\ldots,\mathbold{\eta}^{(t-1)}_M(v_{\ell-1},w)\right)\in\Rb^{s'_t}.
$$
\item [Updating.] Each pair $(v,w)\in V^2$ further updates $\mathbold{m}{}^{(t)}_M(v,w)$
  based on its current label $\mathbold{\eta}^{(t-1)}_M(v,w)$:
$$
\mathbold{\eta}^{(t)}_M(v,w):=\textsc{Upd}^{(t)}_M\left(\mathbold{\eta}^{(t-1)}_M(v,w),\mathbold{m}^{(t)}_M(v,w)\right)\in\Rb^{s_t}.
$$
\end{description}
Here, the message functions $\textsc{Msg}^{(t)}_M$ and update functions
$\textsc{Upd}^{(t)}_M$ are arbitrary functions. 
When a walk \MPNN $M$ only iterates for
a finite number of rounds $T$,  we
define the final labelling $\mathbold{\eta}_M:E\to\Rb^{s}$ with $s=s_T$ returned by $M$ on $G=(V,E,\mathbold{\eta})$, as
$\mathbold{\eta}_M^{\phantom{(t)}}(v,w):=\mathbold{\eta}^{(T)}_M(v,w)$ for every $v,w\in V$. If further aggregation
over the entire graph is needed, e.g., for graph classification, an additional
readout function $\textsc{ReadOut}_M(\ldbl\mathbold{\eta}_M^{\phantom{(t)}}(v,w)\mid v,w\in V\rdbl)$ can be
applied. We ignore the read-out function in this paper as most of the computation happens by means of the message and update functions. We  do comment on read-out functions in Remark~\ref{rem:readout} in Section~\ref{sec:GNNs}.

\section{Upper bound on the expressive power of walk \MPNNs}\label{sec:upperb}
We start by showing that the expressive power of $\ell$-walk \MPNNs is bounded by the expressive power
of $\walk{\ell}$ just as \MPNNs are bounded in expressive power by
 $\wl{1}$. The proof of the following proposition is a straightforward
 modification of the proofs given in \citet{xhlj19} and \citet{grohewl}.

\begin{proposition}\label{prop:upperbw}
For any $\ell$-walk \MPNN $M$, any graph $G=(V,E,\mathbold{\eta})$, and every $t\geq 0$, $\mathbold{\eta}_{\walk{\ell}}^{(t)}\sqsubseteq \mathbold{\eta}_M^{(t)}$. 
\end{proposition}
\begin{proof}
Let $M$ be an $\ell$-walk \MPNN.	
 We verify the proposition by induction on the number of rounds $t$. Clearly, when $t=0$, $\walkrefinementb{\mathbold{\eta}^{(0)}}{\ell}=\mathbold{\eta}=\mathbold{\eta}_M^{(0)}$, so we can focus on $t>0$. Suppose that $\walkrefinementb{\mathbold{\eta}^{(t-1)}}{\ell}\sqsubseteq \mathbold{\eta}_M^{(t-1)}$ holds. We need to show that  $\walkrefinementb{\mathbold{\eta}^{(t)}}{\ell}\sqsubseteq  \mathbold{\eta}_M^{(t)}$ holds as well.

\noindent
Let $v,w,v',w'$ be vertices for which
$\walkrefinementb{\mathbold{\eta}^{(t)}}{\ell}(v,w)=\walkrefinementb{\mathbold{\eta}^{(t)}}{\ell}(v',w')$ is satisfied. By definition of $\walk{\ell}$ this implies that 
$$\bldbl \bigl(\walkrefinementb{\mathbold{\eta}^{(t-1)}}{\ell}(v,v_1),\ldots,\walkrefinementb{\mathbold{\eta}^{(t-1)}}{\ell}(v_{\ell-1},w)\bigr) \brdbl_{v_1,\ldots,v_{\ell-1}\in V}=\bldbl \bigl(\walkrefinementb{\mathbold{\eta}^{(t-1)}}{\ell}(v',v_1),\ldots,\walkrefinementb{\mathbold{\eta}^{(t-1)}}{\ell}(v_{\ell-1},w')\bigr) \brdbl_{v_1,\ldots,v_{\ell-1}\in V},
$$
or in other words, there exists a bijection $\imath:[n]^{\ell-1}\to [n]^{\ell-1}$ such that for every $v_1,\ldots,v_{\ell-1}$ in $V$,
$$\bigl(\walkrefinementb{\mathbold{\eta}^{(t-1)}}{\ell}(v,v_1),\ldots,\walkrefinementb{\mathbold{\eta}^{(t-1)}}{\ell}(v_{\ell-1},w)\bigr)=\bigl(\walkrefinementb{\mathbold{\eta}^{(t-1)}}{\ell}(v',w_1),\ldots,\walkrefinementb{\mathbold{\eta}^{(t-1)}}{\ell}(w_{\ell-1},w')\bigr),
$$
with $(w_1,\ldots,w_{\ell-1})=\imath(v_1,\ldots,v_{\ell-1})$. By induction, this also implies that 
for every $v_1,\ldots,v_{\ell-1}$ there are unique $w_1,\ldots,w_{\ell-1}$ such that
$$\bigl(\mathbold{\eta}^{(t-1)}_M(v,v_1),\ldots,\mathbold{\eta}^{(t-1)}_M(v_{\ell-1},w)\bigr)=\bigl(\mathbold{\eta}^{(t-1)}_M(v',w_1),\ldots,\mathbold{\eta}^{(t-1)}_M(w_{\ell-1},w')\bigr)
$$
holds. This in turn implies that for every $v_1,\ldots,v_{\ell-1}$ there are unique $w_1,\ldots,w_{\ell-1}$ such that
$$
\textsc{Msg}_M^{(t)}\left(\mathbold{\eta}^{(t-1)}_M(v,v_1),\ldots,\mathbold{\eta}^{(t-1)}_M(v_{\ell-1},w)\right)=
\textsc{Msg}_M^{(t)}\left(\mathbold{\eta}^{(t-1)}_M(v',w_1),\ldots,\mathbold{\eta}^{(t-1)}_M(w_{\ell-1},w')\right)
$$
is satisfied. As a consequence,
 $\mathbold{m}_M^{(t)}(v,w)=\mathbold{m}_M^{(t)}(v',w')$ since these are defined by summing up the messages over all $v_1,\ldots,v_{\ell-1}$ and $w_1,\ldots,\allowbreak w_{\ell-1}$, respectively. We also note that if 
$\walkrefinementb{\mathbold{\eta}^{(t)}}{\ell}(v,w)=\walkrefinementb{\mathbold{\eta}^{(t)}}{\ell}(v',w')$ holds, then
$\walkrefinementb{\mathbold{\eta}^{(t-1)}}{\ell}(v,w)=\walkrefinementb{\mathbold{\eta}^{(t-1)}}{\ell}(v',w')$~\citep{lichter2019walk}. Hence
also $\mathbold{\eta}^{(t-1)}_M(v,w)=\mathbold{\eta}^{(t-1)}_M(v',w')$ holds by induction. We may thus conclude that 
$$
\mathbold{\eta}_M^{(t)}(v,w)=\textsc{Upd}^{(t)}_M\left(\mathbold{\eta}^{(t-1)}_M(v,w),\mathbold{m}_M^{(t)}(v,w)\right)=
\textsc{Upd}_M^{(t)}\left(\mathbold{\eta}^{(t-1)}_M(v',w'),\mathbold{m}_M^{(t)}(v',w')\right)=\mathbold{\eta}_M^{(t)}(v',w'),
$$
holds, as desired.
\end{proof}

As already mentioned in the preliminaries, $\walkrefinementb{\mathbold{\eta}^{(t)}}{2}\equiv \wlrefinement{\mathbold{\eta}}^{(t)}$ and $\walkrefinementb{\mathbold{\eta}^{(t\log\lceil \ell\rceil)}}{2}\preceq \walkrefinementb{\mathbold{\eta}^{(t)}}{\ell}$ for all $t\geq 0$. We may thus also infer the following.
\begin{corollary}
For every $\ell$-walk \MPNN $M$, any graph $G=(V,E,\mathbold{\eta})$, and $t\geq 0$, $\wlrefinement{\mathbold{\eta}}^{(t\lceil\log\ell\rceil)}\sqsubseteq\mathbold{\eta}_{M}^{(t)}$.\qed
\end{corollary}

We may thus conclude that for $\ell\geq 2$, $\ell$-walk \MPNNs are limited in their distinguishing power by $\wl{2}$, but they may reach the final labelling $\mathbold{\eta}_{\wl{2}}^{\phantom{(t)}}$ faster than by using $2$-walk \MPNNs. This comes at the cost, however, of a computationally more intensive messaging passing phase.
We next show that $\ell$-walk \MPNNs can also simulate $\walk{\ell}$ from which we can infer
that $\ell$-walk \MPNNs match $\walk{\ell}$ in their expressive power.

\section{Lower bound on the expressive power of $\mathbold{\ell}$-walk \MPNNs}\label{subsec:lowerb}
We next show how to simulate $\walk{\ell}$ by means of $\ell$-walk \MPNNs. In particular, we show that they
can simulate $\walk{\ell}$ on all graphs of a fixed size ($|V|=n$). We provide two simulations, one for when the labels come from a countable domain, and one for when the labels come from an uncountable domain, such as $\Rb^a$ for some $a\in\Nb^+$.

The challenge is to simulate
the hash function used in $\walk{\ell}$ by means of message and update functions, hereby taking into consideration
that $\ell$-walk \MPNNs always perform a sum aggregation over the received messages\footnote{One could also extend the aggregate/combine formalisms used in~\citet{xhlj19} and~\citet{grohewl}. In that formalism, one can define 
$\mathbold{\eta}^{(t)}(i,j):=f_{\textsl{agg}}\left(\ldbl (\mathbold{\eta}^{(t-1)}(i,i_1),\mathbold{\eta}^{(t-1)}(i_1,i_2),\ldots,\mathbold{\eta}^{(t-1)}(i_{\ell_1},j))\mid i_1,\ldots,i_{\ell-1}\in [n] \rdbl\right)$ for some arbitrary aggregate function $f_{\textsl{agg}}$. To simulate $\walk{\ell}$, it then suffices to take $f_{\textsl{agg}}(\cdot)=\textsc{Hash}(\cdot)$ with $\textsc{Hash}(\cdot)$ the hash function used in $\walk{\ell}$. As mentioned already, \MPNNs only allow sum aggregation for $f_{\textsl{agg}}$.
}. 
For the countable case we generalise the technique underlying \GINs~\citep{xhlj19}; for the uncountable case we use multi-symmetric polynomials underlying higher-order graph neural networks \citep{DBLP:conf/nips/MaronBSL19}.

\subsection{Simulating $\walk{\ell}$: Countable case}\label{subsec:countable}
We first consider the setting in which graphs $G=(V,E,\mathbold{\eta})$ have a labelling $\mathbold{\eta}:E\to \Xb$ for some countable domain $\Xb$. Without loss of generality we assume that $\Xb=\Nb$. Indeed, since $\Xb$ is countable the elements in $\Xb$ can be mapped to elements in $\Nb$ by means of an injection. The following result shows that $\ell$-walk \MPNNs can simulate $\walk{\ell}$ on the set of of graphs with $n$ vertices with labels from $\Nb$.

\begin{proposition}\label{prop:lowerb}
For every $\ell\in\Nb$, $\ell\geq 2$, there exists an $\ell$-walk \MPNN $M$  such that
 $\walkrefinementb{\mathbold{\eta}^{(t)}}{\ell}\equiv\mathbold{\eta}^{(t)}_M$ holds for all $t\geq 0$, on any given an input graph ${G = (V, E, \mathbold{\eta})}$ with $\mathbold{\eta}:E\to\Nb$ and $|V|=n$.
\end{proposition}
\begin{proof}
	We define the $\ell$-walk \MPNN $M$ by induction on $t$. More specifically,
	we inductively define the message and update functions of $M$ and verify that $\walkrefinementb{\mathbold{\eta}^{(t)}}{\ell}\equiv\mathbold{\eta}^{(t)}_M$ holds for all $t$ on any
	given an input graph ${G = (V, E, \mathbold{\eta})}$ with $\mathbold{\eta}:E\to\Nb$.
	Furthermore, along the way we verify that for $t>0$, $\mathbold{\eta}^{(t)}_M:E\to \Nb$, i.e., the labels remain to be elements in $\Nb$.
	
	Clearly, by definition, $\walkrefinementb{\mathbold{\eta}^{(0)}}{\ell}=\mathbold{\eta}=\mathbold{\eta}^{(0)}_M$ so we can focus on $t>0$.
	Assume that we have specified $M$ up to round $t-1$ such that $ \walkrefinementb{\mathbold{\eta}^{(t-1)}}{\ell}\equiv\mathbold{\eta}^{(t-1)}_M$ holds, where $\mathbold{\eta}^{(t-1)}_M:E\to\Nb$. We next consider round $t$. 

The labels of a walk of length $\ell$ correspond to an element in $\Nb^\ell$. We want to map these to elements in
$\Nb$ by means of an injection. We can use any pairing function $\tau:\Nb^\ell\to\Nb$ for this purpose\footnote{Since we defined walk \MPNNs over the reals, we assume that $\tau$ extends to
a function $\Rb^\ell\to\Rb$.}. Given such a pairing function, we define 
the function $h:\Nb^\ell\to\Nb$ as
$$
h(a_1,\ldots,a_\ell):=(n^{\ell-1}+1)^{\tau(a_1,\ldots,a_\ell)}.
$$
Then, any multiset $S$ consisting of at most $n^{\ell-1}$ elements in $\Nb^{\ell}$ can be mapped to a number in $\Nb$
by means of the injective function
$$
\varphi(S):=\sum_{(a_1,\ldots,a_{\ell})\in S} h(a_1,\ldots,a_\ell).
$$
Indeed, we here just represent a multiset by its unique $(n^{\ell-1}+1)$-ary representation, just as in \citet{xhlj19}.
It now suffices to define $M$ to consist of the following  message and update functions\footnote{Strictly speaking the message and update functions depend on $n$ which is not allowed by the definition of walk \MPNNs. Since we consider graphs of fixed size, we treat $n$ as as constant. Alternatively, one can incorporate $n$ in the initial labelling and ensure that this value is propagated to all consecutive labellings. In this way, the message and update functions have access to $n$ in every round.} in round $t$: For every $a_1,\ldots,a_\ell\in\Rb$:
 \[
\textsc{Msg}_M^{(t)}(a_1,\ldots,a_{\ell}) :=
h(a_1,\ldots,a_\ell) \in \Rb,
\]
and for every $a,b\in\Rb$,
\[
\textsc{Upd}_M^{(t)}(a,b) := b\in\Rb.\]
It remains to verify that $ \walkrefinementb{\mathbold{\eta}^{(t)}}{\ell}\equiv\mathbold{\eta}^{(t)}_M$ holds. In other words, we need to show that for every
$v,w,v',w'\in V$, 
$$
\walkrefinementb{\mathbold{\eta}^{(t)}}{\ell}(v,w)=\walkrefinementb{\mathbold{\eta}^{(t)}}{\ell}(v',w') \Longleftrightarrow
\mathbold{\eta}^{(t)}_M(v,w)=\mathbold{\eta}^{(t)}_M(v',w').
$$
We define 
for every $v,w\in V$, the multiset
$$
S_{v,w}:=\bldbl \bigl(\mathbold{\eta}^{(t-1)}_M(v,v_1),\ldots,\mathbold{\eta}^{(t-1)}_M(v_{\ell-1},w)\bigr) \brdbl_{v_1,v_2,\ldots,v_{\ell-1}\in [n]}.$$
Hence, 
$\walkrefinementb{\mathbold{\eta}^{(t)}}{\ell}(v,w)=\walkrefinementb{\mathbold{\eta}^{(t)}}{\ell}(v',w')$
if and only if $S_{v,w}=S_{v',w'}$. It now suffices to observe that 
\begin{align*}
	\mathbold{\eta}^{(t)}_M(v,w) &=\sum_{v_1,\ldots,v_{\ell-1}\in V} h\bigl(\mathbold{\eta}^{(t-1)}_M(v,v_1),\ldots,\mathbold{\eta}^{(t-1)}_M(v_{\ell-1},w)\bigr)=\varphi(S_{v,w})
\intertext{ and,}
	\mathbold{\eta}^{(t)}_M(v',w') &=\sum_{v_1,\ldots,v_{\ell-1}\in V} h\bigl(\mathbold{\eta}^{(t-1)}_M(v',v_1),\ldots,\mathbold{\eta}^{(t-1)}_M(v_{\ell-1},w')\bigr)=\varphi(S_{v',w'}).
\end{align*}
Since the multiplicity of every element in the multisets $S_{v,w}$ is bounded by $n^{\ell-1}$,
$\varphi$ is an injection and thus $S_{v,w}=S_{v',w'}$ if and only if $\varphi(S_{v,w})=\varphi(S_{v',w'})$ if and only if
$\mathbold{\eta}^{(t)}_M(v,w)=\mathbold{\eta}^{(t)}_M(v',w')$, 
from which the proposition follows. We note that when the labels assigned by $\mathbold{\eta}^{(t-1)}_M$ belong to $\Nb$, then so do the labels assigned by $\mathbold{\eta}^{(t)}_M$, by the definition of $\varphi$. As a consequence, the $\ell$-walk MPNN $M$ generates labels in $\Nb$ in every round.
\end{proof}

We note that in the simulation above the message and update functions can be fixed, independent of $t$.

\begin{remark}
The function $\varphi$ used in the proof of Proposition~\ref{prop:lowerb} is similar to the one motivating the definition of \GINs~\citep{xhlj19}. The difference is that \citet{xhlj19} incorporate the initial injective mapping from $\Xb$ to $\Nb$ in the first round, and that instead of a representation in $\Nb$, a representation in $\mathbb{Q}$ is used. Translated to our setting this corresponds to defining $h(a_1,\ldots,a_\ell)$ as $(n^{\ell-1}+1)^{-\tau((\imath(a_1),\ldots,\imath(a_\ell))}$ with $\tau$ a pairing function and $\imath:\Xb\to\Nb$ an injection from $\Xb$ to $\Nb$. Since labels now take rational values,
one needs to incorporate an injective mapping from $\mathbb{Q}$ to $\Nb$ in each round $t>1$. By contrast, our simulation produces labels in $\Nb$ for all $t$. \qed
\end{remark}
\begin{remark}
In the standard \MPNN setting, \MPNNs are known to simulate $\wl{1}$ on all graphs with labels in $\Nb$ and that have bounded degree. As such \MPNNs can simulate $\wl{1}$ on graphs of arbitrary size. In our setting, $\walk{\ell}$ assigns labels to all pairs of vertices and the degree is thus always $n$ because the input graphs are complete graphs. Hence, the bounded degree condition reduces to the graphs having a fixed size.\qed
\end{remark}

\subsection{Simulating $\walk{\ell}$: Uncountable case}\label{subsec:uncountable}
We next consider graphs $G=(V,E,\mathbold{\eta})$ with $\mathbold{\eta}:E\to \Rb^{s_0}$ for some $s_0\in\Nb^+$.
We first recall from \citet{DBLP:conf/nips/MaronBSL19} how to, by using multi-symmetric polynomials, assign a unique value in $\Rb^b$ to multisets of $m$ elements in $\Rb^a$ for some $a,b\in\Nb$. Let $a,m\in\Nb$ and let $\mathbold{\alpha}\in [m]^{a}$ be a multi-index, i.e., $\mathbold{\alpha}=(\alpha_1,\ldots,\alpha_{a})$ with $\alpha_i\in[m]$ for $i\in[a]$.
For an element $\mathbold{x}=(x_1,\ldots,x_a)\in\Rb^a$ we write $\mathbold{x}^{\mathbold{\alpha}}:=\prod_{i\in[a]}x_i^{\alpha_i}$ and define $|\mathbold{\alpha}|=\sum_{i\in[a]} \alpha_i$.
Consider a multiset $\mathbold{X}=\ldbl \mathbold{x}_1,\ldots,\mathbold{x}_m\rdbl$ with each $\mathbold{x}_i\in\Rb^a$.
We represent such a multiset by a matrix, also denoted by $\mathbold{X}$, by choosing an arbitrary order on the elements $\mathbold{x}_1,\ldots,\mathbold{x}_m$. More precisely, $\mathbold{X}\in\Rb^{m\times a}$ and $\mathbold{X}_{i\bullet}$
corresponds to one of the $\mathbold{x}_i$'s for each $i\in[m]$. We next define  $p_{\mathbold{\alpha}}(\mathbold{X}):=\sum_{j\in[m]}(\mathbold{X}_{j\bullet})^{\mathbold{\alpha}}$
and let
$
u(\mathbold{X}):=(p_{\mathbold{\alpha}}(\mathbold{X})\mid |\mathbold{\alpha}|\leq m)\in\Rb^b
$, where $b$ corresponds to the number of multi-indexes $\mathbold{\alpha}\in[m]^a$ with $|\mathbold{\alpha}|\leq m$.
More precisely, $b={m+ a \choose a}$.
Then, for $\mathbold{X}$ and $\mathbold{X}'$ in $\Rb^{m\times a}$, $u(\mathbold{X})=u(\mathbold{X}')$ if and only if there exists a permutation $\pi$ of $[m]$ such that
$\mathbold{X}_{j\bullet}=\mathbold{X}_{\pi(j)\bullet}'$ for all $j\in[m]$ (see Proposition 1 in~\citet{DBLP:conf/nips/MaronBSL19}). In other words, by regarding $\mathbold{X}$ and $\mathbold{X}'$ as multisets,
$u(\mathbold{X})=u(\mathbold{X}')$ if and only if $\mathbold{X}$ and $\mathbold{X}'$ represent the same multiset.

\begin{proposition}\label{prop:lowerbreal}
For every $\ell\in\Nb$, $\ell\geq 2$, there exists an $\ell$-walk \MPNN $M$  such that
 $\walkrefinementb{\mathbold{\eta}^{(t)}}{\ell}\equiv\mathbold{\eta}^{(t)}_M$ holds for all $t\geq 0$, on any given an input graph ${G = (V, E, \mathbold{\eta})}$ with $\mathbold{\eta}:E\to\Rb^{s_0}$ and $|V|=n$.
\end{proposition}
\begin{proof}
For each $\ell\in\Nb$, $\ell\geq 2$, $n\in\Nb^+$ and $s_0\in\Nb^+$ we define an
	$\ell$-walk \MPNN $M$ such that $\walkrefinementb{\mathbold{\eta}^{(t)}}{\ell}\equiv\mathbold{\eta}^{(t)}_M$ holds on any given an input graph ${G = (V, E, \mathbold{\eta})}$ with $|V|=n$ and $\mathbold{\eta}:E\to\Rb^{s_0}$.
We define $M$ by induction on $t$. More specifically,
we inductively define the message and update functions of $M$ and verify that $\walkrefinementb{\mathbold{\eta}^{(t)}}{\ell}\equiv\mathbold{\eta}^{(t)}_M$ holds for all $t$ on any
given an input graph ${G = (V, E, \mathbold{\eta})}$ with $|V|=n$ and $\mathbold{\eta}:E\to\Rb^{s_0}$.

Clearly, by definition, $\walkrefinementb{\mathbold{\eta}^{(0)}}{\ell}=\mathbold{\eta}=\mathbold{\eta}^{(0)}_M$ so we can focus on $t>0$.
Assume that we have specified $M$ up to round $t-1$ such that $ \walkrefinementb{\mathbold{\eta}^{(t-1)}}{\ell}\equiv\mathbold{\eta}^{(t-1)}_M$ holds, where $\mathbold{\eta}^{(t-1)}_M:E\to\Rb^{s_{t-1}}$. We next consider round $t$. 

We use the injective function $u$ as described above. We will apply it to the setting whether $a=\ell s_{t-1}$ and $m=n^{\ell-1}$. More precisely, we consider the multi-index set $\{\mathbold{\alpha}\mid \mathbold{\alpha}\in [n^{\ell-1}]^{\ell s_{t-1}},|\mathbold{\alpha}|\leq n^{\ell-1}\}$ of cardinality $s_t={n^{\ell-1}+ \ell s_{t-1} \choose \ell s_{t-1}}$.  We denote the elements in this set by $\mathbold{\alpha}_s$ for $s\in[s_t]$. We define for $\mathbold{x}_1,\ldots,\mathbold{x}_\ell$ in $\Rb^{s_{t-1}}$, 
\begin{equation*}
\textsc{Msg}_M^{(t)}(\mathbold{x}_1,\ldots,\mathbold{x}_\ell):=\bigl((\mathbold{x}_1,\ldots,\mathbold{x}_\ell)^{\mathbold{\alpha}_s} \mid s\in[s_t]\bigr)\in\Rb^{s_t}. \label{eq:msg1}
\end{equation*}
When evaluated on an input graph $G=(V,E,\mathbold{\eta})$, for any $v,w\in V$:
\begin{align*}
\mathbold{m}_M^{(t)}(v,w)&=\sum_{v_1,\dots,v_{\ell-1} \in V}\textsc{Msg}^{(t)}\left(\mathbold{\eta}_M^{(t-1)}(v,v_1),\mathbold{\eta}^{(t-1)}_M(v_1,v_2),\ldots,\mathbold{\eta}^{(t-1)}_M(v_{\ell-1},w)\right)\\
&=\left(p_{\mathbold{\alpha}_s}(\mathbold{X}_{v,w}^{(t-1)})\mid s\in[s_t]\right)=:u^{(t)}(\mathbold{X}^{(t-1)}_{v,w})\in\Rb^{s_t},
\end{align*}
with $\mathbold{X}_{v,w}^{(t-1)}$ the $\Rb^{n^{\ell-1}\times \ell s_{t-1}}$ matrix whose rows are indexed by $(v_1,v_2,\ldots,v_{\ell-1})\in [n]^{\ell-1}$ and such that 
$$
(\mathbold{X}_{v,w}^{(t-1)})_{(v_1,v_2,\ldots,v_{\ell-1}),\bullet}:=\left(\mathbold{\eta}_M^{(t-1)}(v,v_1),\mathbold{\eta}^{(t-1)}_M(v_1,v_2),\ldots,\mathbold{\eta}^{(t-1)}_M(v_{\ell-1},w)\right).
$$
We have mentioned above that $u^{(t)}(\mathbold{X}_{v,w}^{(t-1)})=u^{(t)}(\mathbold{X}_{v',w'}^{(t-1)})$ if and only if 
\begin{multline*}
\bldbl \bigl(\mathbold{\eta}_M^{(t-1)}(v,v_1),\mathbold{\eta}^{(t-1)}_M(v_1,v_2),\ldots,\mathbold{\eta}^{(t-1)}_M(v_{\ell-1},w)\bigr) \brdbl_{v_1,v_2,\ldots,v_{\ell-1}\in [n]}=\\\bldbl \bigl(\mathbold{\eta}_M^{(t-1)}(v',v_1),\mathbold{\eta}^{(t-1)}_M(v_1,v_2),\ldots,\mathbold{\eta}^{(t-1)}_M(v_{\ell-1},w')\bigr) \brdbl_{v_1,v_2,\ldots,v_{\ell-1}\in [n]}.
\end{multline*}
From the induction hypothesis we know that $\walkrefinementb{\mathbold{\eta}^{(t-1)}}{\ell}\equiv \mathbold{\eta}_M^{(t-1)}$ and hence,
$\mathbold{m}_M^{(t)}({v,w})=\mathbold{m}_M^{(t)}(v',w')$ if and only if
 $u^{(t)}(\mathbold{X}_{v,w}^{(t-1)})=u^{(t)}(\mathbold{X}_{v',w'}^{(t-1)})$ if and only
if $\walkrefinementb{\mathbold{\eta}^{(t)}}{\ell}(v,w)=\walkrefinementb{\mathbold{\eta}^{(t)}}{\ell}(v',w')$.
It now suffices to define for any $\mathbold{x}\in\Rb^{s_{t-1}}$ and $\mathbold{y}\in\Rb^{s_{t}}$:
$$
\textsc{Upd}_M^{(t)}(\mathbold{x},\mathbold{y}):=\mathbold{y}\in\Rb^{s_t},
$$
such that when evaluated on the input graph, $$\mathbold{\eta}_M^{(t)}(v,w):=\textsc{Upd}_M^{(t)}(\mathbold{\eta}_M^{(t-1)}({v,w}),\mathbold{m}_M^{(t)}(v,w)):=\mathbold{m}_M^{(t)}(v,w)=u^{(t)}(\mathbold{X}^{(t-1)}_{v,w}).$$
Hence, $\mathbold{\eta}_M^{(t)}({v,w})=\mathbold{\eta}_M^{(t)}(v',w')$ if and only if $\walkrefinementb{\mathbold{\eta}^{(t)}}{\ell}(v,w)=\walkrefinementb{\mathbold{\eta}^{(t)}}{\ell}(v',w')$, as desired.
\end{proof}

We remark that the dimensions $s_t$ needed in each round grow very fast. For example, for $n=10$, $\ell=2$ and with initial $s_0$ set to $1$, we have $s_1={12\choose 2}=66$ 
and $s_2={142\choose 132}=664\,226\,242\,466\,073$. This is in sharp contrast to $s_t=1$ for all $t\geq 1$ in the countable case. It is an interesting open problem how to treat the real number case with constant (or small) $s_t$ dimensions. A preliminary investigation in dimensionality aspects of representations of multi-sets with elements from the reals is reported in~\citet{DBLP:conf/icml/WagstaffFEPO19} and~\citet{seo2019discriminative}.

\section{Simulation of $\walk{\ell}$ by graph neural networks}\label{sec:GNNs}
We next propose a couple of learnable  graph neural network (\GNN) models, each of which can be seen as walk \MPNNs, and match $\walk{\ell}$ in expressive power. More specifically, we first revisit the \GNNs proposed by \citet{grohewl} and show that they can simulate $\walk{\ell}$ on a specific input graph.
We next propose two \GNNs, one close in spirit to the \GINs of~\citet{xhlj19}, and one close
to the higher-order \GNNs of~\citet{DBLP:conf/nips/MaronBSL19}, and show that these can simulate $\walk{\ell}$ on all graphs of a fixed size $n$ provided that all labels combined form a finite set.
The latter restriction is to ensure that the approximations of functions by \MLPs 
are injective. This is needed since these approximations will substitute the hash functions used in $\walk{\ell}$).

\subsection{Simulating $\walk{\ell}$ on a single graph}
When dealing with a single graph, we can generalise the \GNN architecture from~\citet{grohewl}. As already mentioned in the introduction, \cite{grohewl} propose higher-order \GNNs that, in order to simulate $\wl{2}$, require maintaining
$\mathcal{O}(n^3)$ many labels. It would be desirable to need at most 
$\mathcal{O}(n^2)$ many dimensions. We achieve this by allowing non-linear layers in the \GNN architecture. We show the simulation for $\ell=2$, and thus $\wl{2}$, but the construction can be generalised easily to $\walk{\ell}$ for 
arbitrary $\ell\geq 2$. We comment on this generalisation later in this section.

Let $G=(V,E,\mathbold{\eta})$ be the  input graph  with $\mathbold{\eta}:E\to \Rb^{s_0}$ for some $s_0\in\Nb^+$. We represent $G$ by means of a tensor $\mathbf{A}\in \Rb^{n^2\times s_0}$ such that $\mathbf{A}_{ij\bullet}=\mathbold{\eta}(i,j)\in\Rb^{s_0}$ for each $i,j\in[n]$.

Similarly to~\citet{grohewl} we say that a tensor $\mathbf{A}\in\Rb^{n^2\times s}$ for $s\in\Nb^+$
is label-independent modulo equality\footnote{\citet{grohewl} use the notion of row-independence modulo equality because the row dimension correspond to the labels. Since we work with tensors, we use label-independence modulo equality. We always assume that the labels are stored in the last dimension in the tensors.} if the set of unique labels $\{ \mathbf{A}_{ij\bullet}\mid i,j\in[n]\}$ consists of linearly independent vectors in $\Rb^s$. We remark that the input tensor $\mathbf{A}$ can always be assumed to satisfy this property by hot-one encoding the labelling $\mathbold{\eta}$. In the worst case, we need labels in $\Rb^{n^2}$ to do so\footnote{This is really worst case as it corresponds to every pair in $[n]^2$ to have a distinct label.}. 
Generalising the \GNNs from~\cite{grohewl}, we
define $\mathbf{A}^{(0)}:=\mathbf{A}$ and for $t>0$ we define a tensor $\mathbf{A}^{(t)}\in\Rb^{n^2\times s_t}$ for some $s_t\in\Nb^+$ based on a tensor $\mathbf{A}^{(t-1)}\in\Rb^{n^2\times s_{t-1}}$. More specifically, for $i,j\in[n]$ and $s\in[s_t]$:
\begin{equation}
\mathbf{A}^{(t)}_{ijs}:=\mathsf{ReLU}\left( \sum_{k\in[n]}\sum_{c,d\in [s_{t-1}]}\mathbf{A}^{(t-1)}_{ikc}\mathbf{A}^{(t-1)}_{kjd}
\mathbf{W}^{(t)}_{cds} - q \mathbf{J}_{ijs}\right), \label{eq:2GNN}
\end{equation}
where $\mathbf{W}^{(t)}\in\Rb^{s_{t-1}\times s_{t-1}\times s_t}$ is a (learnable) weight matrix,
$\mathbf{J}\in\Rb^{n^2\times s_t}$ is a tensor consisting of all ones, and $q$ is (learnable) scalar in $\Rb$. 
A first observation is that
we can cast the update rules~(\ref{eq:2GNN}) as a $2$-walk \MPNN $M$. Indeed, it suffices to define for each $t>0$ and $\mathbf{a},\mathbf{b}\in\Rb^{s_{t-1}}$:
$$
\textsc{Msg}_M^{(t)}(\mathbf{a},\mathbf{b}):= \left(\sum_{c,d\in[s_{t-1}]}\mathbf{a}_c\mathbf{b}_d\mathbf{W}^{(t)}_{cds} - q\mid s\in[s_t]\right)\in\Rb^{s_t}
$$
and for $\mathbf{a}\in \Rb^{s_{t-1}}$ and $\mathbf{b}\in\Rb^{s_t}$:
$$
\textsc{Upd}_M^{(t)}(\mathbf{a},\mathbf{b}):=\mathsf{ReLU}(\mathbf{b}).
$$
Clearly, when using these message and update functions the corresponding $2$-walk \MPNN $M$ will compute the labelling
$\mathbold{\eta}^{(t)}_M(i,j):=\mathbf{A}_{ij\bullet}^{(t)}$ for $i,j\in [n]$ and each $t\geq 0$. Proposition~\ref{prop:upperbw} thus implies that the 
expressive power of architectures of the form~(\ref{eq:2GNN}) are bounded by $\wl{2}$. We next provide a matching lower bound.

\begin{proposition}
For a given input graph $G=(V,E,\mathbold{\eta})$ with $\mathbold{\eta}:E\to\Rb^{s_0}$ and such that 
the labels are label-independent modulo equality,
 there exists a scalar $q\in\Rb$ and, for each $t>0$, there exists a weight tensor $\mathbf{W}^{(t)}$ such $\mathbold{\eta}_{\wl{2}}^{(t)}\equiv \mathbf{A}^{(t)}$ holds, where $\mathbf{A}^{(t)}$ is defined as in ~(\ref{eq:2GNN}). Furthermore, $\mathbf{A}^{(t)}$ is label-independent modulo equality.
\end{proposition}
\begin{proof}
The base case $t=0$ is satisfied by assumption. We next assume that the induction hypothesis is satisfied for $t-1$ and consider round $t$. We define $\mathbf{W}^{(t)}$ as the product of a number of tensors, which we define next.
In a similar way as in~\cite{grohewl} we first map each label in $\mathbf{A}^{(t-1)}$ to a canonical vector encoding that label. Here, with a canonical vector we mean a binary vector with precisely one occurrence of the value $1$. Intuitively,
the canonical vector in which $1$ appears in position $i$ corresponds to the $i$th label, relative to some ordering on the labels. By the induction hypothesis, $\mathbf{A}^{(t-1)}$ is label-independent modulo equality. So, if we assume that there are $c_{t}$ distinct labels $\mathbf{a}_1,\ldots,\mathbf{a}_{c_t}$ in 
$\mathbf{A}^{(t-1)}$, then we know that these vectors in $\Rb^{s_{t-1}}$ are linearly independent. We denote by $\mathsf{uniq}(\mathbf{A}^{(t-1)})$ the $c_t\times s_{t-1}$
matrix consisting of these distinct labels. Due to linear independence, there exists a $s_{t-1}\times c_t$ matrix 
$\mathbf{V}^{(t)}$ such that $\mathsf{uniq}(\mathbf{A}^{(t-1)})\mathbf{V}^{(t)}=\mathbf{Id}\in\Rb^{c_t\times c_t}$. As a first step, we multiply each label in $\mathbf{A}^{(t-1)}$ with $\mathbf{V}^{(t)}$. More specifically, we define a tensor
$\mathbf{B}\in\Rb^{n^2\times c_t}$ such that for $i,j\in[n]$ and $c\in [c_t]$:
$$
\mathbf{B}^{(t)}_{ijc}:= \sum_{s'\in [s_{t-1}]}\mathbf{A}^{(t-1)}_{ijs'}\mathbf{V}^{(t)}_{s'c}.
$$
In other words, for $i,j\in[n]$ and $c\in [c_t]$:
$$
\mathbf{B}^{(t)}_{ijc}=\begin{cases}
1 & \text{if $(\mathbf{A}^{(t-1)})_{ij\bullet}=\mathbf{a}_{c}$}\\
0 &\text{otherwise}.
\end{cases}
$$
We next define a tensor $\mathbf{C}^{(t)}\in\Rb^{n^2\times c_t\times c_t}$ such that for $i,j\in[n]$, $c,d\in [c_t]$: $$
\mathbf{C}^{(t)}_{ijcd}:=\sum_{k\in[n]}
\mathbf{B}^{(t)}_{ikc}
\mathbf{B}^{(t)}_{kjd}.
$$
As a consequence,  for $i,j\in[n]$, $c,d\in [c_t]$:
$$
\mathbf{C}^{(t)}_{ijcd}=N(i,j,c,d):=\text{number of $k\in[n]$ such that $\mathbf{A}^{(t-1)}_{ik\bullet}=\mathbf{a}_{c}$ and
$\mathbf{A}^{(t-1)}_{kj\bullet}=\mathbf{a}_{d}$}.
$$
By the induction hypothesis, $\mathbold{\eta}_{\wl{2}}^{(t-1)}\equiv\mathbf{A}^{(t-1)}$, and we may assume that there is a bijection between the unique labels in $\mathbf{A}^{(t-1)}$ and those assigned by $\mathbold{\eta}_{\wl{2}}^{(t-1)}$. Let us assume that this bijection is such that for $c\in[c_t]$, $\mathbf{a}_{c}$
corresponds to the label $\ell_c$. We remark that $
\mathbf{C}^{(t)}_{ij\bullet\bullet}=\mathbf{C}^{(t)}_{i'j'\bullet\bullet}\in \Rb^{c_t\times c_t}$ if and only if 
for each pair of labels $\mathbf{a}_{c}$ and $\mathbf{a}_{d}$ in $\mathbf{A}^{(t-1)}$, 
$$
N(i,j,c,d)=N(i',j',c,d).
$$
By the induction hypothesis, this in turn is equivalent to
$$
|\{ k\in[n]\mid \mathbold{\eta}_{\wl{2}}^{(t-1)}(i,k)=\ell_{c},\mathbold{\eta}_{\wl{2}}^{(t-1)}(k,j)=\ell_{d}\}|
=
|\{ k\in[n]\mid \mathbold{\eta}_{\wl{2}}^{(t-1)}(i',k)=\ell_{c},\mathbold{\eta}_{\wl{2}}^{(t-1)}(k,j')=\ell_{d}\}|,
$$
for every $c,d\in[c_t]$.
Since this holds for any pair of labels $\ell_{c}$ and $\ell_{d}$, this is equivalent to $\mathbold{\eta}_{\wl{2}}^{(t)}(i,j)=\mathbold{\eta}_{\wl{2}}^{(t)}(i',j')$. So, at this point we already know that 
$\mathbold{\eta}_{\wl{2}}^{(t)}\equiv\mathbf{C}^{(t)}$. In what follows, we turn 
$\mathbf{C}^{(t)}$ into a tensor in $\Rb^{n^2\times s_t}$ which is label-independent modulo equality whilst preserving equivalence to 
$\mathbold{\eta}_{\wl{2}}^{(t)}$.

The first thing we do is to turn each of the labels in $\mathbf{C}^{(t)}_{ij\bullet\bullet}$ into a single number in $\Nb^+$. Similarly as in~\cite{grohewl} we identify the maximum entry $\mathsf{max}$ in $\mathbf{C}^{(t)}$ and define the vector $\mathbf{M}^{(t)}\in\Rb^{c_t}$ such that for 
$d\in[c_t]$:
$$
\mathbf{M}^{(t)}_{d}:=(\mathsf{max}+1)^{d-1}.$$
We next define the tensor $\mathbf{D}^{(t)}\in\Rb^{n^2\times c_t}$ such that for $i,j\in[n]$ and $c\in[c_t]$:
$$
\mathbf{D}^{(t)}_{ijc} := \sum_{d\in[c_t]}\mathbf{C}^{(t)}_{ijcd}\mathbf{M}^{(t)}_{d}.
$$
In other words,
$$\mathbf{D}^{(t)}_{ijc}= \sum_{d\in[c_t]} N(i,j,c,d) (\mathsf{max}+1)^{d-1}
$$
and we thus  have represented each vector $\mathbf{C}^{(t)}_{ijc\bullet}$ by its $(\mathsf{max}+1)$-ary representation.
Since all $N(i,j,c,d)\leq \mathsf{max}$, we have that 
$\mathbf{C}^{(t)}_{ijc\bullet}=\mathbf{C}^{(t)}_{i'j'c\bullet}$ if and only if 
$\mathbf{D}^{(t)}_{ijc}=\mathbf{D}^{(t)}_{i'j'c}$ and thus $\mathbf{C}^{(t)}_{ij\bullet\bullet}=\mathbf{C}^{(t)}_{i'j'\bullet\bullet}$ if and only if 
$\mathbf{D}^{(t)}_{ijc\bullet}=\mathbf{D}^{(t)}_{i'j'\bullet}$. As a consequence, also $\mathbold{\eta}_{\wl{2}}^{(t)}\equiv\mathbf{D}^{(t)}$ holds. We perform the same reduction once more, but this time using the maximum entry $\mathsf{max}'$ in $\mathbf{D}^{(t)}$. That is, we define 
$\mathbf{N}^{(t)}$ just like $\mathbf{M}^{(t)}$ but using $\mathsf{max}'$ instead of $\mathsf{max}$ and consider the matrix
$\mathbf{E}^{(t)}\in \Rb^{n\times n}$ such that for $i,j\in[n]$:
$$
\mathbf{E}^{(t)}_{ij}:=\sum_{c\in[c_t]}\mathbf{D}^{(t)}_{ijc}\mathbf{N}^{(t)}_{c}.
$$
A similar argument as before shows that $\mathbold{\eta}_{\wl{2}}^{(t)}\equiv\mathbf{E}^{(t)}$. Let $e_1> e_2> \cdots > e_{s_t}$ be the unique values in $\mathbf{E}^{(t)}$. By construction of $\mathbf{E}^{(t)}$, each $e_s>0$. We next consider the
vector $\mathbf{U}^{(t)}\in\Rb^{s_t}$ define as $\mathbf{U}^{(t)}_{s}=\frac{1}{e_s}$ for $s\in[s_t]$. We define the tensor
$\mathbf{F}^{(t)}\in\Rb^{n^2\times s_t}$ such that
for $i,j\in[n]$ and $s\in [s_t]$:
$$
\mathbf{F}^{(t)}_{ijs}:=\mathbf{E}^{(t)}_{ij}\mathbf{U}^{(t)}_{s}.
$$
In other words, for $i,j\in[n]$ and $s\in[s_t]$:
$$
\mathbf{F}^{(t)}_{ijs}= \frac{\mathbf{E}^{(t)}_{ij}}{e_s}.
$$
We are now ready to define $\mathbf{W}^{(t)}\in\Rb^{s_{t-1}\times s_{t-1}\times s_t}$. Indeed, for 
$c,d\in[s_{t-1}]$ and $s\in[s_t]$:
$$\mathbf{W}^{(t)}_{cds}:= \sum_{c',d'\in[c_t]}\mathbf{V}^{(t)}_{cc'}\mathbf{V}^{(t)}_{dd'}\mathbf{M}^{(t)}_{d'}\mathbf{N}^{(t)}_{c'}\mathbf{U}^{(t)}_{s}.$$
Hence, we can write $\mathbf{F}^{(t)}$ as 
$$
\mathbf{F}^{(t)}_{ijs}=\sum_{k\in[n]}\sum_{c,d\in[s_{t-1}]}
\mathbf{A}^{(t-1)}_{ikc}
\mathbf{A}^{(t-1)}_{kjd}\mathbf{W}^{(t)}_{cds}.
$$
It remains to identify a scalar $q$ such that 
$$
\mathbf{A}^{(t)}_{ijs}=\mathsf{ReLU}\left( \sum_{k\in[n]}\sum_{c,d\in [s_{t-1}]}\mathbf{A}^{(t-1)}_{ikc}\mathbf{A}^{(t-1)}_{kjd}
\mathbf{W}^{(t)}_{cds} - q \mathbf{J}_{ijs}\right)
$$
is label-independent modulo equality.
To this aim, let $q^{(t)}$ be the greatest value in $\mathbf{F}^{(t)}$ smaller than $1$ and consider
the tensor $\mathbf{G}^{(t)}\in\Rb^{n^2\times s_t}$ such that for $i,j\in[n]$ and $s\in[s_t]$,
$\mathbf{G}^{(t)}_{ijs}:=\mathbf{F}^{(t)}_{ijs}-q^{(t)}$. Hence,
$$
\mathbf{G}^{(t)}_{ijs}=\begin{cases}
1-q^{(t)} &\text{if $\mathbf{E}^{(t)}_{ij}=e_s$}\\
> 0 &\text{if $\mathbf{E}^{(t)}_{ij}>e_s$}\\
\leq 0 &\text{if $\mathbf{E}^{(t)}_{ij}< e_s$}.
\end{cases}
$$
Hence, for $i,j\in[n]$, $s\in[s_t]$:
$$
\mathbf{A}^{(t)}_{ijs}:=\mathsf{ReLU}\left(\mathbf{G}^{(t)}_{ijs}\right)= \begin{cases}
1-q^{(t)} &\text{if $\mathbf{E}^{(t)}_{ij}=e_s$}\\
> 0 &\text{if $\mathbf{E}^{(t)}_{ij}>e_s$}\\
0 &\text{if $\mathbf{E}^{(t)}_{ij}< e_s$}.
\end{cases}
$$
It is again easily verified that $\mathbold{\eta}_{\wl{2}}^{(t)}\equiv\mathbf{A}^{(t)}$.
Indeed, $\mathbold{\eta}_{\wl{2}}^{(t)}\sqsubseteq\mathbf{A}^{(t)}$ follows immediately from  $\mathbold{\eta}_{\wl{2}}^{(t)}\equiv\mathbf{E}^{(t)}$. To show  $\mathbf{A}^{(t)}\sqsubseteq\mathbold{\eta}_{\wl{2}}^{(t)}$ it suffices to observe that  when $\mathbf{A}^{(t)}_{ij\bullet}=\mathbf{A}^{(t)}_{i'j'\bullet}$ holds, these vectors contain $1-q^{(t)}$
at the same (unique) position, say at position $s\in[s_t]$. Hence, $\mathbf{E}_{ij}^{(t)}=e_s=\mathbf{E}_{i'j'}^{(t)}$ and again due to $\mathbold{\eta}_{\wl{2}}^{(t)}\equiv\mathbf{E}^{(t)}$, $\mathbold{\eta}_{\wl{2}}^{(t)}(i,j)=\mathbold{\eta}_{\wl{2}}^{(t)}(i',j')$.

The unique labels in $\mathbf{A}^{(t)}$ are also linearly independent. To see this, we note that the unique labels correspond to the unique elements $e_1,e_2,\ldots,e_{s_t}$ in $\mathbf{E}^{(t)}$. As a consequence, the value $e_s$ corresponds to the label
$$
(0,\ldots, 0, \underbrace{1-q^{(t)}}_{\text{position $s$}}, >0 , \ldots, >0)
$$
in $\mathbf{A}^{(t)}$.
In other words, these form (up to a permutation) an upper-triangular matrix with $1-q^{(t)}\neq 0$ on its diagonal, and this is known to be a non-singular matrix. As a consequence, the unique labels in $\mathbf{A}^{(t)}$ are linearly independent.

We further observe that $q^{(t)}$ can be chosen to be any number satisfying
$\frac{{n}^{{n^2}^{n^2}}-1}{{n}^{{n^2}^{n^2}}}<q<1$.
This follows from upper bounding $\mathsf{max}$ by $n$, and $c_t$ by $n^2$ which results in
 an upper bound for $\max'$ as $n^{n^2}$. Hence, $\frac{{n}^{{n^2}^{n^2}}-1}{{n}^{{n^2}^{n^2}}}$ is an upper bound on the largest value in $\mathbf{F}^{(t)}$ smaller than $1$ for any $t>0$. As a consequence, $q^{(t)}$ can be chosen uniformly across al layers. All combined, this shows that architectures of the form of~(\ref{eq:2GNN}) can simulate $\wl{2}$ on $G=(V,E,\mathbold{\eta})$.
\end{proof}

\begin{remark}
To generalise the construction to simulate $\walk{\ell}$ for $\ell>2$ it suffices to consider $\ell-1$ matrix multiplications of $\mathbf{A}^{(t-1)}$ in the architecture~(\ref{eq:2GNN}) and to extend the weight tensor to be of dimensions $\ell s_{t-1}\times s_t$. The construction of $\mathbf{W}^{(t)}$ is entirely similar, with the exception that the matrix $\mathbf{E}^{(t)}$, which will now be in $\Rb^{n^\ell}$, is obtained by encoding each of its $\ell$ dimensions as a number in $\Nb$. So instead of only two matrices $\mathbf{M}^{(t)}$ and $\mathbf{N}^{(t)}$, we need $\ell$ such matrices. Finally,
$q$ is lower bounded by $\left(n\underbrace{{}^{{{n^2}^{\cdots}}^{n^2}}}_{\text{$\ell$ times}}-1\right)/n\underbrace{{}^{{{n^2}^{\cdots}}^{n^2}}}_{\text{$\ell$ times}}$. \qed
\end{remark}

We note that we can use \GNNs of the form~(\ref{eq:2GNN}) to distinguish graphs by simply running the \GNN on the direct sum of the two graphs, just as for $\wl{1}$.

\subsection{Simulating $\walk{\ell}$ on a collection of graphs}
We have seen two different ways of simulating $\walk{\ell}$ by $\ell$-walk \MPNNs in Section~\ref{subsec:lowerb}. We next turn these simulations into learnable \GNNs by replacing the message functions by multi layer perceptrons (\MLPs), just as in
 \citep{xhlj19} and~\citep{DBLP:conf/nips/MaronBSL19}. \MLPs are known to approximate any continuous bounded function. In order to approximate the message functions by \MLPs we need to
ensure that the message functions are continuous and that the approximations returned by
the \MLPs inherit the crucial injectivity properties (on multisets) of the functions being approximated.

Let us first consider the simulation presented in Section~\ref{subsec:countable}. In that simulation we used an arbitrary pairing function $\tau:\Nb^{\ell}\to \Nb$ and defined 
$\textsc{Msg}^{(t)}(a_1,\ldots,a_\ell):=(n^{(\ell-1)}+1)^{\tau(a_1,\ldots,a_\ell)}$. To ensure
continuity we choose  $\tau:\Nb^{\ell}\to \Nb: (a_1,a_2,\ldots,a_\ell)\mapsto p_1^{a_1}p_2^{a_2}\cdots p_{\ell}^{a_\ell}$ with $p_i$ the $i$th prime number. Clearly, its extension $\tau:\Rb^{\ell}\to \Rb:(x_1,\ldots,x_\ell)\mapsto 2^{x_1}3^{x_2}\cdots p_{\ell}^{x_\ell}$ is a continuous function and similarly, $h:\Rb^{\ell}\to\Rb:(x_1,\ldots,x_\ell)\to (n^{(\ell-1)}+1)^{\tau(x_1,\ldots,x_\ell)}$ is continuous.
We  remark that other continuous pairing functions $\Nb^\ell\to\Nb$ can be used instead.

There are now various ways of using \MLPs to approximate $h$ and $\tau$. We recall that the simulation in 
Section~\ref{subsec:countable} concerns graphs $G=(V,E,\mathbold{\eta})$ with $\mathbold{\eta}:E\to\Nb\subseteq\Rb$. Hence, we can represent $G$ by means of a matrix $\mathbf{A}^{(0)}$ such that $(\mathbf{A}^{(0)})_{ij}:=\mathbold{\eta}(i,j)$ for all $i,j\in[n]$. We then define
for $t>0$, the matrix $\mathbf{A}^{(t)}\in \Rb^{n\times n}$, as follows:
\begin{equation}
(\mathbf{A}^{(t)})_{ij}:= \sum_{i_1,\ldots,i_\ell\in[n]} \mathsf{MLP}_{\mathbold{\theta}^{(t)}}\bigl((\mathbf{A}^{(t-1)})_{ii_1},(\mathbf{A}^{(t-1)})_{i_1i_2},\ldots,(\mathbf{A}^{(t-1)})_{i_{\ell-1}j}\bigr),\label{eq:GIN1}
\end{equation}
where $\mathsf{MLP}_{\mathbold{\theta}^{(t)}}:\Rb^\ell\to \Rb$ is an \MLP
with parameters $\mathbold{\theta}^{(t)}$. The \MLP is to be trained to approximate $h$, just as for  \GINs~\citep{xhlj19}.
Alternatively, we can define $\mathbf{A}^{(t)}\in \Rb^{n\times n}$, as follows:
\begin{equation}
\mathbold{A}^{(t)}_{ij}:=\sum_{i_1,\ldots,i_{\ell-1}\in [n]} \mathsf{MLP}_{\mathbold{\theta}^{(t)}}\left(\mathsf{MLP}_{\mathbold{\theta}^{(t)}_1}(\mathbold{A}^{(t-1)})_{ii_1}\cdot \mathsf{MLP}_{\mathbold{\theta}^{(t)}_2}(\mathbold{A}^{(t-1)})_{i_1i_2}\cdots
\mathsf{MLP}_{\mathbold{\theta}^{(t)}_\ell}(\mathbold{A}^{(t-1)})_{i_{\ell-1}j}\right),\label{eq:GIN2} \end{equation}
where $\mathsf{MLP}_{\mathbold{\theta}^{(t)}}:\Rb\to \Rb$  is an \MLP
with parameters $\mathbold{\theta}^{(t)}$ used to approximate the function $x\to (n^{(\ell-1)}+1)^x$ and $\mathsf{MLP}_{\mathbold{\theta}^{(t)}_i}:\Rb\to \Rb$, for $i\in[\ell]$, is an \MLP
with parameters $\mathbold{\theta}^{(t)}_i$ used to approximate the function $x\to p_i^x$ with $p_i$ the $i$th prime number.
Yet another alternative could be to encode $G$ as the tensor $\mathbf{A}^{(0)}\in\Rb^{n\times n\times \ell}$ with $(\mathbf{A}^{(0)})_{ijs}:=p_s^{\mathbold{\eta}(i,j)}$ with $p_s$ the $s$th prime number, for $s\in[\ell]$, and then define for $t>0$, the tensor 
$\mathbf{A}^{(t)}\in\Rb^{n\times n\times \ell}$ with for $i,j\in[n]$ and $s\in[\ell]$:
\begin{equation}
\mathbold{A}^{(t)}_{ijs}:=\mathsf{MLP}_{\theta^{(t)}_1}\left(
\sum_{i_1,\ldots,i_{\ell-1}\in [n]} \mathsf{MLP}_{\theta^{(t)}_2}\left((\mathbold{A}^{(t-1)})_{ii_11}\cdot (\mathbold{A}^{(t-1)})_{i_1i_22}\cdots
(\mathbold{A}^{(t-1)})_{i_{\ell-1}j\ell}\right)\right),\label{eq:GIN3}
\end{equation}
where $\mathsf{MLP}_{\mathbold{\theta}^{(t)}_1}:\Rb\to \Rb^\ell$  is an \MLP
with parameters $\mathbold{\theta}^{(t)}_1$ used to approximate the function $x\to (2^x,3^x,\ldots,p_\ell^x)$
and $\mathsf{MLP}_{\mathbold{\theta}^{(t)}_2}:\Rb\to \Rb$ is an \MLP
with parameters $\mathbold{\theta}^{(t)}_2$ used to approximate the function $x\to (n^{(\ell-1)}+1)^x$.

In all three formulations the \MLPs have to be learned based on the available labels. In general, one could approximate the functions up to arbitrary precision provided that the set of labels belong to some compact set. We also need, however, to ensure injectivity. One way to guarantee this is by assuming that only a finite number of labels are present in the collection of graphs~\citep{DBLP:conf/nips/MaronBSL19,Sato2020ASO}.
We thus can guarantee the following.
\begin{proposition}
For each $n,\ell,t\in\Nb$ with $\ell\geq 2$, there exists parameters of the \MLPs in~(\ref{eq:GIN1}),~(\ref{eq:GIN2}) and~(\ref{eq:GIN3}), such that $\mathbf{A}^{(t)}\equiv\mathbold{\eta}_{\walk{\ell}}^{(t)}$ for any graph $G=(V,E,\mathbold{\eta})$ with $|V|=n$, $\mathbold{\eta}:E\to \Gamma\subseteq \Nb$, where
$\Gamma$ is a finite set of numbers.
\end{proposition}

We can proceed in a similar way using the simulation given in Section~\ref{subsec:uncountable}.
As already observed by~\citet{DBLP:conf/nips/MaronBSL19} for $\wl{2}$, one can decompose the function $u^{(t)}$ in that simulation as a product of $\ell$ other functions. More precisely,
let $G=(V,E,\mathbold{\eta})$ with $|V|=n$ and $\mathbold{\eta}:E\to\Rb^{s_0}$ for some $s_0\in\Nb^+$. We encode $G$ as a tensor $\mathbf{A}^{(0)}\in\Rb^{n^2\times s_0}$ as before. Then for $t>0$, assume that $\mathbf{A}^{(t-1)}\in\Rb^{n^2\times s_{t-1}}$
and define
\begin{equation}
\mathbold{A}^{(t)}_{ijs}:=\sum_{i_1,\ldots,i_{\ell-1}\in[n]} g^{(t)}_1(\mathbold{A}^{(t-1)})_{ii_1s}\cdot g^{(t)}_2(\mathbold{A}^{(t-1)})_{i_1i_2s}\cdots
g^{(t)}_\ell(\mathbold{A}^{(t-1)})_{i_{\ell-1}js} \label{eq:GNN5}
\end{equation}
for continuous functions $g_p^{(t)}:\mathbb{R}^{s_{t-1}}\to\mathbb{R}^{s_t}$, for $p\in[\ell]$, which we define next. Consider again the multi-index set $\{\mathbold{\alpha}\mid \mathbold{\alpha}\in [n^{\ell-1}]^{\ell s_{t-1}},|\mathbold{\alpha}|\leq n^{\ell-1}\}$ of cardinality $s_t={n^{\ell-1}+ \ell s_{t-1}\choose \ell s_{t-1}}$ used in the simulation of $\walk{\ell}$ in Section~\ref{subsec:uncountable}. We can represent each  multi-index $\mathbold{\alpha}_s$ in this set, for $s\in[s_t]$, in  the form 
$
(\mathbold{\alpha}^1_s,\ldots,\mathbold{\alpha}^\ell_s)$ where for $j\in[\ell]$,
$\mathbold{\alpha}^j_s\in [n^{\ell-1}]^{s_{t-1}}$ and furthermore, $\sum_{j\in[\ell]} |\mathbold{\alpha}^j_s|\leq n^{\ell-1}$.
We next define for $p\in[\ell]$, $g_p^{(t)}:\mathbb{R}^{s_{t-1}}\to\mathbb{R}^{s_t}$ such that for $\mathbold{x}\in\Rb^{s_{t-1}}$,
$$
g_p^{(t)}(\mathbold{x}):=(\mathbold{x}^{\mathbold{\alpha}^p_s}\mid s\in[s_t])\in\Rb^{s_t}.
$$
Hence, for $\mathbold{x}_1,\ldots,\mathbold{x}_\ell\in\Rb^{s_{t-1}}$ we have
$$
\prod_{p=1}^{\ell} g_p^{(t)}(\mathbold{x}_p)=\bigl((\mathbold{x}_1,\ldots,\mathbold{x}_\ell)^{\mathbold{\alpha}_s}\mid s\in[s_t]\bigr)\in\Rb^{s_t},
$$
which precisely corresponds to the message function used in Section ~\ref{subsec:uncountable}.
As a consequence, $\mathbold{A}^{(t)}$ as defined in~(\ref{eq:GNN5}) is equivalent to $\mathbold{\eta}_{\walk{\ell}}^{(t)}$.
To turn~(\ref{eq:GNN5}) into a learnable graph neural network we define
\begin{equation}
\mathbold{A}^{(t)}_{ijs}:=\sum_{i_1,\ldots,i_{\ell-1}\in[n]} \mathsf{MLP}_{\mathbold{\theta}^{(t)}_1}(\mathbold{A}^{(t-1)})_{ii_1s}\cdot \mathsf{MLP}_{\mathbold{\theta}^{(t)}_2}(\mathbold{A}^{(t-1)})_{i_1i_2s}\cdots
\mathsf{MLP}_{\mathbold{\theta}^{(t)}_\ell}(\mathbold{A}^{(t-1)})_{i_{\ell-1}js}, \label{eq:nonlGNN}
\end{equation}
where for $p\in[\ell]$, $\mathsf{MLP}_{\mathbold{\theta}^{(t)}_p}:\Rb^{s_{t-1}}\to \Rb^{s_t}$ is a multi layer perceptron applied to the labels in $\mathbf{A}^{(t-1)}$. More specifically, $\mathsf{MLP}_{\mathbold{\theta}^{(t)}_p}(\mathbold{A}^{(t-1)})_{ijs}:=(\mathsf{MLP}_{\mathbold{\theta}^{(t)}_p}(\mathbold{A}_{ij\bullet}^{(t-1)}))_s$ for $p\in[\ell]$ and $s\in [s_{t}]$. Furthermore, for $p\in[\ell]$, $\mathsf{MLP}_{\mathbold{\theta}^{(t)}_p}$ is used to approximate the function $g_p^{(t)}$.  We may thus conclude that:
\begin{proposition}
For each $n,\ell,t\in\Nb$ with $\ell\geq 2$, there exists parameters of the \MLPs in~(\ref{eq:nonlGNN}) such that $\mathbf{A}^{(t)}\equiv\mathbold{\eta}_{\walk{\ell}}^{(t)}$ for any graph $G=(V,E,\mathbold{\eta})$ with $|V|=n$, $\mathbold{\eta}:E\to \Gamma\subseteq \Rb^{s_0}$, where
$\Gamma$ is a finite set of real vectors.
\end{proposition}

	We note that the graph neural network models~(\ref{eq:GIN1}),~(\ref{eq:GIN2}),~(\ref{eq:GIN3}),~(\ref{eq:GNN5})  and~(\ref{eq:nonlGNN}) can all be cast as $\ell$-walk MPNNs, which implies that their expressive power is bounded by $\walk{\ell}$ as well.

\begin{remark}
	We remark that the second-order non-linear invariant \GNNs proposed in \citet{DBLP:conf/nips/MaronBSL19} are a special case of~(\ref{eq:nonlGNN}) by
	letting $\ell=2$, and  hence they are bounded by $\wl{2}$ in expressive power. We observe that allowing for multiple matrix multiplications in second-order \GNNs, as in~(\ref{eq:nonlGNN}), does not increase expressive power. Instead, it may only result in a faster convergence towards the final $\wl{2}$ labelling. This partially answers a question raised in \citet{openprob} related to the impact  of polynomial layers on the expressive power of higher-order invariant \GNNs. \qed
\end{remark}

\begin{remark}
If one desires to start from a vertex-labeled graph, one can add an initialisation step in the \GNNs which converts the graph into and edge-labeled graph, as explained in Remark~\ref{remark:edgevsvertex}. Furthermore,
	this initialisation step can be performed by tensor computations as shown in~\citet{DBLP:conf/nips/MaronBSL19}.\qed
	\end{remark}
\begin{remark}\label{rem:readout}
So far, we only considered the expressive power of walk \MPNNs related to  distinguishing edges (or pairs of vertices to be more precise).
As mentioned earlier, we may also use walk \MPNNs to distinguish graphs. In the setting of walk \MPNNs this corresponds to 
running the walk \MPNN for multiple rounds $T$ and then use a read-out function $\textsc{ReadOut}$ on the obtained multiset of labels. More precisely, for $\walk{\ell}$, two graphs $G=(V,E,\mathbold{\eta})$
and $H=(V',E',\mathbold{\eta}')$ with $\mathbold{\eta}:E\to\Sigma$ and $\mathbold{\eta}':E'\to\Sigma$ are
said to be indistinguishable after round $T$ if
$$
\bldbl \mathbold{\eta}^{(T)}_{\walk{\ell}}(i,j)\mid i,j\in[n]\brdbl=
\bldbl (\mathbold{\eta}')^{(T)}_{\walk{\ell}}(i,j)\mid i,j\in[n]\brdbl
$$
holds.
Hence, to check whether this equality holds, it suffices to consider a read-out function which
assigns a unique value in $\Nb$ to multisets of  elements in $\Nb$ of size $n^2$, for the case when labels are in $\Nb$, and a unique value in $\Rb^b$, for some $b\in\Nb^+$, to multisets of elements in $\Rb^{s_T}$ of size $n^2$, for the case when labels are reals. Alternatively, one can define a read-out function which assigns to each possible label a unique basis vector in $\Rb^b$, and then simply sum these up to create a histogram. In each of these cases, an additional \MLP can be used to approximate such a read-out function, as described in \citet{DBLP:conf/nips/MaronBSL19}.  \qed
\end{remark}

\section{Conclusion}\label{sec:conclude}
We introduced $\ell$-walk \MPNNs as a general formalism for iteratively
constructing graph embeddings based on walks of length $\ell$ between pairs of vertices.
In terms of expressive power, $\ell$-walk \MPNNs match with the walk refinement procedure $\walk{\ell}$ of \citet{lichter2019walk}. When $\ell=2$,
this procedure coincides with $\wl{2}$ and as such, $2$-walk \MPNNs are equally expressive as $\wl{2}$. In fact, $\ell$-walk \MPNNs are also bounded in expressive power by $\wl{2}$ but can possibly distinguish graphs faster because more information is taken into account in each iteration. 
We provide a number of concrete learnable \GNNs, all of which can be cast as $\ell$-walk \MPNNs. These \GNNs use non-linear layers and only require $\mathcal{O}(n^2)$ many embeddings. All proposed \GNNs are equally expressive as $\walk{\ell}$ and $\wl{2}$ in particular. It would be interesting to see how the proposed \GNNs perform in practice.

% \bibliographystyle{plainnat}
% \bibliography{refs}
\end{document}